\documentclass[runningheads]{llncs}
\usepackage{graphicx}

\usepackage[utf8]{inputenc}
\usepackage[misc]{ifsym}
\usepackage{url}
\usepackage[hidelinks]{hyperref}
\usepackage{csquotes}
\usepackage{color}
\usepackage{amsmath,amssymb}
\usepackage{cite}
\usepackage{subcaption}

\usepackage{xspace}
\usepackage{tikz}
\usepackage{pgfplots}
\pgfplotsset{compat=newest}

\usepackage{cleveref}
\crefname{equation}{}{}
\Crefname{equation}{}{}
\crefname{figure}{Fig.}{Figs.}
\Crefname{figure}{Figure}{Figures}

\newcommand{\mi}[1]{\ensuremath{\mathit{#1}}}

\newcommand{\slv}[1]{\textsc{#1}} 
\newcommand{\alphaslv}{\slv{Alpha}\xspace}
\newcommand{\clasp}{\slv{Clasp}}

\renewcommand{\orcidID}[1]{\href{https://orcid.org/#1}{\hbox{\includegraphics[height=10pt]{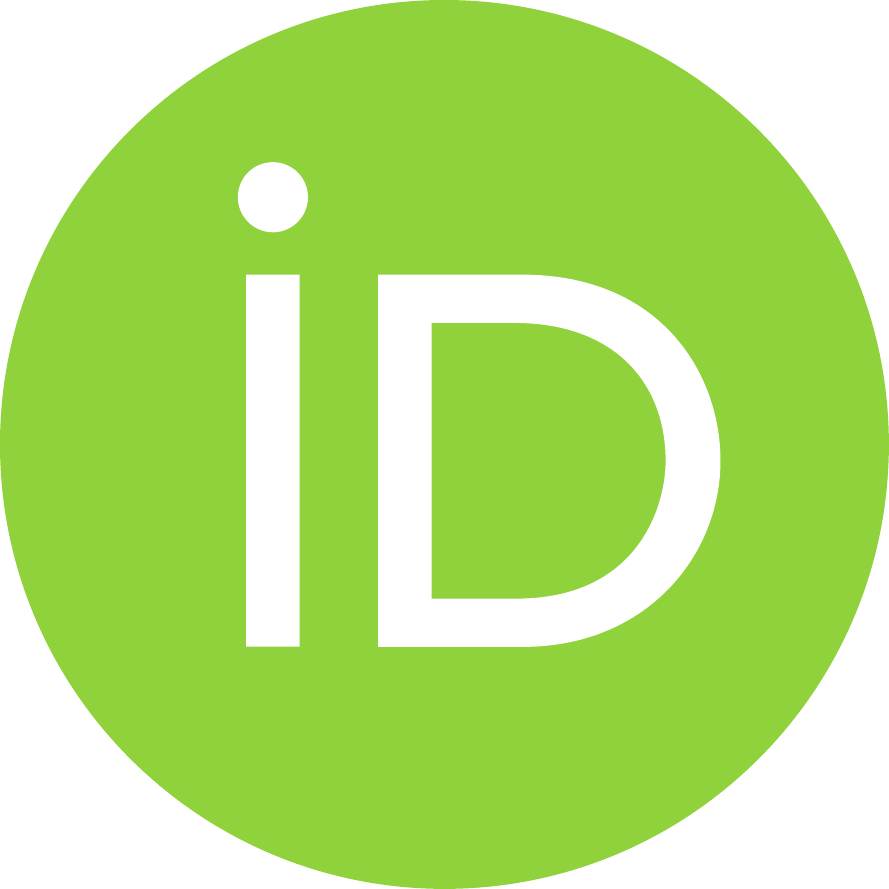}}}}

\newcommand{\prog}{\ensuremath{P}}
\newcommand{\atoms}{\ensuremath{\mathcal{A}}}

\newcommand{\assignment}{\ensuremath{A}}

\newcommand{\assignmentp}{\ensuremath{A}^+}
\newcommand{\assignmentm}{\ensuremath{A}^-}

\newcommand{\preds}{\ensuremath{\mathcal{P}}}
\newcommand{\nafsymbol}{\ensuremath{not}}
\newcommand{\naf}[1]{\ensuremath{\nafsymbol~#1}}
\newcommand{\head}{\ensuremath{\mi{H}}}
\newcommand{\body}{\ensuremath{\mi{B}}}
\newcommand{\bodyp}{\ensuremath{\mi{B}^+}}
\newcommand{\bodyn}{\ensuremath{\mi{B}^-}}
\newcommand{\ground}{\ensuremath{\mi{grd}}}

\newcommand{\facts}{\ensuremath{\mathit{facts}}}
\newcommand{\pred}{\ensuremath{\mathit{pred}}}
\newcommand{\heads}{\ensuremath{\mathit{heads}}}
\newcommand{\HB}{\ensuremath{\mathcal{HB}}}
\newcommand{\As}{\ensuremath{\mathcal{A}}}
\newcommand{\gratoms}{\ensuremath{\mi{At}_\mi{grd}}}
\newcommand{\consts}{\ensuremath{\mathcal{C}}}
\newcommand{\variables}{\ensuremath{\mathcal{V}}}

\newcommand{\gmem}{\mi{G}}
\newcommand{\gmemprime}{\mi{G}'}
\newcommand{\gmemdprime}{\mi{G}''}
\newcommand{\gstrat}{\mi{gs}}
\newcommand{\gsdefault}{\ensuremath{\gstrat_{\mi{def}}}}
\newcommand{\gskunassigned}{\ensuremath{\gstrat}}
\newcommand{\gstrataccu}{\ensuremath{\gstrat^{\mi{accu}}}}
\newcommand{\gsdefaultaccu}{\ensuremath{\gstrataccu_{\mi{def}}}}
\newcommand{\gskunassignedaccu}{\ensuremath{\gstrataccu}}
\newcommand{\vars}{\ensuremath{\mi{vars}}}

\def\hyph{-\penalty0\hskip0pt\relax}
\begin{document}
\title{Degrees of Laziness in Grounding\thanks{The final authenticated publication is available online at \url{https://doi.org/10.1007/978-3-030-20528-7_22}}}
\subtitle{Effects of Lazy-Grounding Strategies on ASP Solving}

%
%
\author{Richard Taupe\inst{1,2}\textsuperscript{(\Letter)}\orcidID{0000-0001-7639-1616} \and Antonius Weinzierl\inst{3}\orcidID{0000-0003-2040-6123} \and Gerhard Friedrich\inst{2}\orcidID{0000-0002-1992-4049}}
\authorrunning{R. Taupe et al.}
%
\institute{
Siemens AG Österreich\\
\email{richard.taupe@siemens.com}\\
\and
Alpen-Adria-Universität Klagenfurt\\
\email{gerhard.friedrich@aau.at} \and
Technische Universität Wien, Institut für Logic and Computation, KBS Group\\
\email{weinzierl@kr.tuwien.ac.at}
}
\maketitle              
\begin{abstract}
The traditional ground-and-solve approach to Answer Set Programming (ASP) suffers from the grounding bottleneck, which makes large-scale problem instances unsolvable.
Lazy grounding is an alternative approach that interleaves grounding with solving and thus uses space more efficiently.
The limited view on the search space in lazy grounding poses unique challenges, however, and can have adverse effects on solving performance.
In this paper we present a novel characterization of degrees of laziness in grounding for ASP, i.e. of compromises between lazily grounding as little as possible and the traditional full grounding upfront.
We investigate how these degrees of laziness compare to each other formally as well as, by means of an experimental analysis using a number of benchmarks, in terms of their effects on solving performance.
Our contributions are the introduction of a range of novel lazy grounding strategies, a formal account on their relationships and their correctness, and an investigation of their effects on solving performance.
Experiments show that our approach performs significantly better than state-of-the-art lazy grounding in many cases.

\keywords{Answer Set Programming \and Lazy grounding \and Heuristics.}
\end{abstract}
\section{Introduction}
\label{intro}

Answer Set Programming (ASP) \cite{aspbook-gelfond,aspbook-baral,Gebser.2012,Gelfond.1988}
is a declarative knowledge representation formalism
that has been applied in a variety of industrial and scientific applications.
The success of ASP is rooted in efficient solvers
such as  \slv{clingo} \cite{DBLP:conf/lpnmr/GebserKK0S15} or \slv{DLV} \cite{Leone.2006}, which apply
the \emph{ground-and-solve} approach, i.e. they first instantiate the given non-ground program and then apply a number of efficient solving techniques to find the answer sets of the variable-free (i.e., ground) program.

This approach suffers from the \emph{grounding bottleneck} since in many practical and industrial applications the ground program is too large to fit in memory.
Problem instances in industrial applications can be quite large and cannot be grounded by modern grounders such as \slv{gringo} \cite{DBLP:conf/lpnmr/GebserKKS11} or \slv{I-DLV} \cite{DBLP:journals/ia/CalimeriFPZ17} in acceptable time and/or space \cite{Eiter.2008}.

Lazy-grounding ASP systems such as
\slv{gasp} \cite{gasp}, \slv{ASPeRiX} \cite{Lefevre.2017}, \slv{OMiGA} \cite{omiga_system}, and most recently
\alphaslv \cite{alpha_technical} successfully avoid the grounding bottleneck by
interleaving grounding and solving, but suffer from substandard search performance.
For practical applications one can now decide between running out of memory with a ground-and-solve system, or running out of time with a lazy-grounding system.
Since the grounding bottleneck is an inherent issue of the ground-and-solve approach, improvements of lazy-grounding ASP solving are an important contribution for dealing with large, real-world problem instances.

Therefore, we equipped \alphaslv with state-of-the art heuristics successfully employed by other ASP solvers, namely MOMs \cite{Pretolani.1993} for initialization of heuristic scores and VSIDS \cite{Moskewicz.2001} for their dynamic modification.
Both have been implemented in a similar fashion as in \slv{clasp} \cite{clasp_journal}.
Somewhat surprisingly, however, those heuristics improved performance of lazy-grounding solving by a much smaller degree than expected.
A subsequent investigation revealed that lazy grounding does not provide sufficient information on the search space for such heuristics to perform adequately, because by grounding lazily the solver has only a limited view on the search space.
This is a novel challenge for ASP solving, which traditional ground-and-solve ASP solvers did not have to face.

In order to improve solving performance this work investigates ways to offset the limited view of the search space in lazy-grounding ASP solving.
We explore various lazy-grounding strategies to find compromises between full upfront grounding and largely blind search heuristics.
In summary, the contributions of this work are:
\begin{itemize}
	\item the introduction of a field of novel lazy-grounding strategies for ASP evaluation,
  \item a formal investigation of how these grounding strategies compare to each other and to previously known ones, as well as
	\item an experimental analysis in terms of their effects on solving performance, showing that our approach is able to perform significantly better than state-of-the-art lazy grounding in many cases.
\end{itemize}

Outline:
After preliminaries in \Cref{sec:preliminaries}, novel lazy-grounding strategies are introduced in \Cref{sec:heuristics} and their relationships are formally investigated.
\Cref{sec:experiments} presents experimental results, and \Cref{sec:conclusion} concludes.

\section{Preliminaries}
\label{sec:preliminaries}

Let $\consts$ be a finite set of constants, $\variables$ be a set of variables and $\preds$ be a finite set of predicates.
A (classical) atom is an expression $p(t_1,\dots,t_n)$ where $p$ is an $n$-ary predicate and $t_1,\dots,t_n \in \consts \cup \variables$ are terms, and a literal is either an atom $a$ or its default negation $\naf{a}$.
An Answer-Set Program \prog\ is a finite set of (normal) rules of the form
\begin{equation}
  \label{eqRule}
	h \leftarrow b_1,~\ldots~,b_m,~\naf{b_{m+1}},~\ldots,~\naf{b_n}.
\end{equation}
where $h$ and $b_1,\dots,b_m$ are positive literals (i.e. atoms) and $\naf{b_{m+1}}$, $\ldots$, $\naf{b_n}$ are negative literals.
Given a rule $r$, we denote by $\head(r)=\{h\}$,
$\body(r) = \{ b_1,\dots,b_m,$ $not~b_{m+1},$
$\ldots,not~b_n \}$,
$\bodyp(r) = \{ b_1, \dots, b_m \}$, and
$\bodyn(r) = \{ b_{m+1},$ $\dots,$ $b_n \}$
the head, body, positive body, and negative body of $r$, respectively.
If $\head(r) = \emptyset$, $r$ is a called a constraint, and a fact if $\body(r) = \emptyset$.
Given a literal $l$, set of literals $L$, or rule $r$, we denote by $\vars(l)$, $\vars(L)$, or $\vars(r)$ the set of variables occurring in $l$, $L$, or $r$, respectively. 
A literal $l$ or rule $r$ is ground if $\vars(l) = \emptyset$ or $\vars(r) = \emptyset$, respectively.  The set of all ground atoms is denoted by $\gratoms$. A program $\prog$ is ground if all its rules $r \in \prog$ are. 
As usual, in the remainder of this work we only consider safe programs $\prog$, where each rule $r \in \prog$ is safe, i.e., each variable occurring in $r$ also occurs in its positive body, formally, $\vars(r) \subseteq\vars(\bodyp(r))$.
The function $\pred \colon 2^\atoms \to 2^\preds$ maps a set of atoms to their predicates, e.g. $\pred(\{ a(1,2), a(X,Y) \}) = \{ a/2 \}$.
The set $\heads(\prog) = \{ \head(r) \mid r \in \prog \}$ contains the heads of all rules in $\prog$.

An (Herbrand) interpretation $I$ is a subset of the Herbrand base w.r.t.~$P$,
i.e., $I \subseteq \gratoms$. An interpretation $I$ satisfies a literal $l$,
denoted $I \models l$, if $l \in I$ for positive $l$ and $l \notin I$ for
negative $l$. $I$ satisfies a ground rule $r$, denoted $I \models r$, if
$\bodyp(r) \subseteq I \land \bodyn(r) \cap I = \emptyset$ implies
$\head(r) \subseteq I$ and $\head(r) \neq \emptyset$. Given an interpretation
$I$ and a ground program $\prog$, the FLP-reduct $\prog^I$ of $P$ w.r.t.~$I$ is
the set of rules $r \in \prog$ whose body is satisfied by $I$, i.e.,
$\prog^I = \{ r \in \prog \mid \bodyp(r) \subseteq I \land \bodyn(r) \cap I =
\emptyset \}$. $I$ is an \emph{answer set} of a ground program $\prog$ if $I$
is the subset-minimal model of $P^I$.

A substitution $\sigma: \variables \to \consts$ is a mapping of variables to
constants. Given an atom $\mi{at}$ the result of applying a substitution
$\sigma$ to $\mi{at}$ is denoted by $\mi{at}\sigma$; this is extended in the
usual way to rules $r$, i.e., $r\sigma$ for a rule of the above form is
$\mi{h}\sigma \leftarrow \mi{b}_1\sigma, \ldots, b_m\sigma, \naf b_{m+1}\sigma, \naf
\mi{b}_n\sigma$. The grounding of a rule is given by
$\ground(r) = \{ r\sigma \mid \sigma \text{ is a substitution for all }v \in
\vars(r)\}$ and the grounding $\ground(\prog)$ of a program $\prog$ is given by
$\ground(\prog) = \bigcup_{r \in \prog} \ground(r)$.
The answer sets of a non-ground program $\prog$ are given
by the answer sets of $\ground(\prog)$.

Computing all answer sets such that $\ground(\prog)$ is constructed lazily is typically done by a loop composed of two phases: given a partial assignment (that is initially empty), first ground those rules that potentially fire under the current assignment, second expand the current assignment (using propagation and guessing). If the loop reaches a fixpoint, i.e., no more rules potentially fire and nothing is left to propagate or guess on, and no constraints are violated, then the current assignment is an answer set (cf. \cite{alpha_technical,Leutgeb.2017} for a detailed account of the \alphaslv ASP system).
A (partial) \emph{assignment} $\assignment$ is a set of signed atoms where $\assignmentp$ denotes the atoms assigned a positive value and $\assignmentm$ those assigned a negative value in $\assignment$.
Note that for this work it is sufficient to consider $\assignment$ to be Boolean (while the solving component of \alphaslv also considers a third and positive truth value must-be-true).
Given an assignment $\assignment$, a ground rule $r\sigma$ stemming from a non-ground rule $r \in \prog$ and a substitution $\sigma$, if $\bodyp(r\sigma) \subseteq \assignmentp$ holds then $r\sigma$ is \emph{of interest} w.r.t.~$\assignment$ and must be grounded, because $r\sigma$ potentially fires under $\assignment$.
Given two assignments $A, A'$ we define the combination $A \uplus A' = B$ to be an assignment such that $B^+ = A^+ \cup A'^+$ and $B^- = A^- \cup A'^-$.

\section{Lazy-Grounding Strategies}
\label{sec:heuristics}
Currently, a ground rule is only returned to the solver if it is of interest, i.e., if its positive body is fully satisfied.
This is a very restrictive grounding strategy in order to save space and avoid the grounding bottleneck.
As experience shows, this \emph{maximally strict} grounding strategy employed by \alphaslv results in non-optimal search performance, because state-of-the-art search procedures, derived from propositional SAT solving, only operate on grounded parts of the problem.
With maximally strict lazy-grounding these search procedures (most importantly branching heuristics) are left mostly blind, because a large part of the given problem instance simply has not been grounded yet.

In the following we thus investigate more permissive lazy-grounding strategies that lie between the maximally strict one and the full upfront grounding (the \emph{maximally permissive} grounding strategy).
The more permissive a grounding strategy, the less restrictions it poses on ground rules produced by the grounder. Thus, ground rules are produced earlier and in higher quantity, which allows search procedures to be more informed about the problem at hand.

\begin{definition}
  Let $\prog$ be an answer-set program, $\As$ be the set of assignments, $G_m = 2^{\gratoms}$ be the set of possible grounder memories, and $R \subseteq \prog$ the set of rules of $\prog$ that are not ground.
  Then, a \emph{lazy-grounding strategy} is a function $s: \As \times G_m \times R \to G_m \times 2^{\ground(P)}$ mapping a triple of assignment, grounder memory, and a rule with variables to a new grounder memory and a set of ground instances of the rule, i.e.,  $(\assignment, \gmem, r) \mapsto (\gmemprime, R')$ with $R' \subseteq \ground(r)$.
\end{definition}

Observe that a grounder memory $\gmem \subseteq 2^{\gratoms}$ is a subset of the Herbrand base $\HB_P = \gratoms$ and thus can be seen as one half of an assignment, i.e., either $\assignmentp$ or $\assignmentm$.
Since rules in ASP must be safe, a grounding substitution for all variables of the positive body of a rule is also a grounding substitution for the whole rule.
Therefore, it is sufficient to consider only the positive body for lazy grounding.

Considering both negative and positive body atoms could allow a more restrictive grounding than currently employed in \alphaslv, because a grounding instantiation could be rejected if one of the negative body literals is currently true. This approach, however, would require the solver to ground additional rules also when backtracking in the search, because backtracking removes assignments and those could then lead to negative body atoms no longer being true. Thus in \alphaslv grounding only considers the positive body of a rule and we follow this choice here.
In the remainder of this work we therefore identify a grounder memory $\gmem$ with an assignment $\assignment$ such that $\assignmentp = \gmem$ and $\assignmentm = \emptyset$, i.e., a grounder memory identifies a fully positive assignment.

In order to avoid ground instantiations of rules that can never be applicable we introduce a notion for deterministically inactive rules.
Intuitively, a rule is inactive if it contains a positive literal over a predicate that does not occur in any rule head (or fact) and hence cannot be derived, or if it contains a negative literal that also occurs as a fact in the program hence its negation never holds.
Formally, given a ground rule $r \in \ground(\prog)$, $r$ is \emph{inactive} if there exists $a \in \bodyp(r)$ with $\pred(a) \notin \pred(\heads(\prog))$ or $a \in \bodyn(r)$ with $a \in \atoms(\facts(\prog))$.\footnote{The notion of inactive rule could be generalized to cover more rules, but we decidedly chose a syntactic condition that is easy to check algorithmically.}

The formalization of \alphaslv's default grounding strategy is as follows.

\begin{definition}
  The \emph{default grounding strategy} for a program $\prog$ is a lazy-grounding strategy $\gsdefault(\assignment, \gmem, r) = (\gmemprime,R)$ such that $\gmemprime = \assignmentp$ and $R = \{ r^\prime \in \ground(r) \mid$ $r^\prime \text{ is }$ $\text{not inactive and of interest w.r.t.~} A\}$.
\end{definition}

The following notion helps to characterize a class of grounding strategies that are at least as permissive as the maximally strict strategy and strictly less permissive than the maximally permissive strategy.

\begin{definition}
\label{def:weakly_applicable}
A ground rule $r \in \ground(\prog)$ is \emph{weakly applicable} w.r.t. an assignment $\assignment$ if
$\bodyp(r) \cap \assignmentm = \emptyset$
and
$r$ is not inactive.
\end{definition}
Intuitively, a ground rule $r$ is weakly applicable if is not inactive and no positive body atom is assigned false.

Given an assignment $\assignment$, a (non-ground) rule $r$, and a substitution $\sigma$ such that $r\sigma$ is ground, we call the set $L_\assignment$ of positive literals of $r$ whose grounding is in $\assignment$, i.e., $L_\assignment = \{ l \in \bodyp(r) \mid l\sigma \in \assignmentp\}$, the \emph{assigned literals} of $r\sigma$ w.r.t.~$\assignment$; furthermore, if $\vars(L_\assignment) = \vars(r)$ we say $r\sigma$ is \emph{all-variable-assigning} w.r.t.~$\assignment$.

\begin{definition}
\label{def:n_unassigned}
A ground instance $r\sigma$ of a non-ground rule $r \in \prog$ is \emph{$k$-unassigned} w.r.t. an assignment $\assignment$ if it is weakly applicable, its set $L_\assignment$ of assigned literals is all-variable-assigning, and $| \bodyp(r) \setminus L_\assignment | \leq k$, i.e. at most $k$ literals in the positive body of $r\sigma$ are still unassigned.
\end{definition}

For grounding strategies based on $k$-unassignedness, we further distinguish between constraints and normal rules, because these two types affect the search procedure in different ways (as \Cref{sec:experiments} shows).
A modified grounder then returns all ground rules that can be produced w.r.t. the current partial assignment and that are $k_{co}$-unassigned in the case of constraints or $k_{ru}$-unassigned in the case of other rules, where $k_{co}$ and $k_{ru}$ are parameterizable.
Values $k_{co} = k_{ru} = 0$ yield the maximally strict grounding strategy, i.e., a rule is $0$-unassigned if and only if it is of interest.
The field of novel grounding strategies then is as follows.
\begin{definition}
  The \emph{$k$-unassigned grounding strategy} for a program $\prog$ is a lazy-grounding strategy $\gskunassigned_{k_{co},k_{ru}}(\assignment, \gmem, r) = (\gmemprime,R)$ such that $\gmemprime = \assignmentp$ and $R = \{ r^\prime \in \ground(r) \mid \head(r^\prime) = \emptyset, r^\prime \text{ is $k_{co}$-unassigned w.r.t.~} A \} \cup \{  r^\prime \in \ground(r) \mid \head(r^\prime) \neq \emptyset, r^\prime$ is $k_{ru}$-unassigned w.r.t.~$A\}$.
\end{definition}

Strategies with $k_{co} > k_{ru}$ ground more constraints than rules,
allowing better-informed search heuristics and at the same time fewer superfluous ground rules.
Intuitively, these grounding strategies yield a larger grounding in each step of a lazy-grounding solver, but they are still limited to only yield ground instances of rules that are very close to the current search path, since $k$\hyph{}unassignedness requires all variables to be bound by instances in the current assignment.

To give the grounder more freedom such that ground instances can be obtained that are further away from the current search path, we introduce accumulator grounding strategies. The core idea is to use the grounder memory to store ground atoms that were encountered earlier in another branch of the search for answer sets but are not necessarily true in the current branch of the search.

\begin{definition}
  The \emph{default accumulator grounding strategy} for a program $\prog$ is a lazy-grounding strategy $\gsdefaultaccu(\assignment, \gmem, r) = (\gmemprime,R)$ such that $\gmemprime = \gmem \cup \assignmentp$ and $R = \{ r^\prime \in \ground(r) \mid r^\prime \text{ is not inactive and of interest w.r.t.~} \gmemprime \uplus \assignment\}$.
\end{definition}

Using such an accumulator the grounder is able to obtain ground instances resulting from a combination of different search paths.
The accumulator can also be added to the $k$-unassigned grounding strategy as follows.

\begin{definition}
  The \emph{$k$-unassigned accumulator grounding strategy} for a program $\prog$ is a lazy-grounding strategy $\gskunassignedaccu_{k_{co},k_{ru}}(\assignment, \gmem, r) = (\gmemprime,R)$ such that $\gmemprime = \gmem \cup \assignmentp$ and $R = \{ r \in \ground(r) \mid \head(r) = \emptyset, r \text{ is $k_{co}$-unassigned w.r.t.~} \gmemprime \uplus \assignment\} \cup \{  r^\prime \in \ground(r) \mid \head(r^\prime) \neq \emptyset, r^\prime \text{ is $k_{ru}$-unassigned w.r.t.~} \gmemprime \uplus \assignment \}$.
\end{definition}

\paragraph{Relationships Between Lazy-Grounding Strategies.}
Some of the lazy-grounding strategies introduced above are subsumed by others, i.e., the sets of ground rules produced by some grounding strategies are subsets of those produced by others.
First, each $k$-unassigned grounding strategy is subsumed by a $k+1$-unassigned grounding strategy, intuitively because a $k$-unassigned rule also is a $k+1$-unassigned rule. Formally, and more detailed:
\begin{proposition}
  \label{prop:gskunassigned_subsumes_greaterk}
  Given an assignment $\assignment$, a grounding memory $\gmem$, and a rule $r$.
  Let $\gskunassigned_{k_{co},k_{ru}}(\assignment, \gmem, r) = (\gmemprime,R)$ and $\gskunassigned_{k'_c,k'_r}(\assignment, \gmem, r) = (\gmemdprime,R')$, then $R \subseteq R'$ for any $k_{co}' \geq k_{co}$ and $k_{ru}' \geq k_{ru}$.
\end{proposition}

\begin{proof}
  Let $\gskunassigned_{k_{co},k_{ru}}(\assignment, \gmem, r) = (\gmemprime,R)$ and $r' \in R$. Then $r'$ is either a $k_{co}$-unassigned constraint or a $k_{ru}$-unassigned rule and because $k_{co}' \geq k_{co}$ and $k_{ru}' \geq k_{ru}$ it follows that $r'$ is either a $k_{co}'$-unassigned constraint or a $k_{ru}'$-unassigned rule, respectively. In either case it holds that $r' \in R'$ for $\gskunassigned_{k'_c,k'_r}(\assignment, \gmem, r) = (\gmemdprime,R')$.
\end{proof}

Second, each $k$-unassigned strategy subsumes the default grounding strategy.
\begin{proposition}
  \label{prop:gskunassigned_subsumes_gsdefault}
  Given an assignment $\assignment$, a grounding memory $\gmem$, and a rule $r$.
  Let $\gsdefault(\assignment, \gmem, r) = (\gmemprime,R)$ and $\gskunassigned_{k_{co},k_{ru}}(\assignment, \gmem, r) = (\gmemdprime,R')$, then $R \subseteq R'$ for any $k_{co},k_{ru} \geq 0$.
\end{proposition}

\begin{proof}
  Let $\gsdefault(\assignment, \gmem, r) = (\gmemprime,R)$ and $r \in R$, then $r$ is not inactive and of interest w.r.t.~$\assignment$, i.e., $\bodyp(r) \subseteq \assignmentp$.
  By the latter, it holds that $r$ is $0$-unassigned and consequently $r \in R'$ for $\gskunassigned_{0,0}(\assignment, \gmem, r) = (\gmemdprime,R')$.
  From Proposition \ref{prop:gskunassigned_subsumes_greaterk} it then follows that $r \in R'$ for any $\gskunassigned_{k_{co},k_{ru}}(\assignment, \gmem, r) = (\gmemdprime,R')$ with $k_{co},k_{ru} \geq 0$.
\end{proof}

Third, the accumulator variant of a grounding strategy subsumes the grounding strategy without accumulator.
\begin{proposition}
  \label{prop:gsaccu_subsumes_gs}
  For an assignment $\assignment$, a grounding memory $\gmem$, and a rule $r$:
  \begin{enumerate}
  \item if $\gsdefault(\assignment, \gmem, r) = (\gmemprime,R)$ and $\gsdefaultaccu(\assignment, \gmem, r) = (\gmemdprime,R')$ then $R \subseteq R'$.
  \item if $\gskunassigned_{k_{co},k_{ru}}(\assignment, \gmem, r) = (\gmemprime,R)$ and $\gskunassignedaccu_{k_{co},k_{ru}}(\assignment, \gmem, r) = (\gmemdprime,R')$ then $R \subseteq R'$ for any $k_{co}, k_{ru} \geq 0$.
  \item if $\gskunassignedaccu_{k_{co},k_{ru}}(\assignment, \gmem, r) = (\gmemprime,R)$ and $\gskunassignedaccu_{k'_{co},k'_{ru}}(\assignment, \gmem, r) = (\gmemdprime,R')$ then $R \subseteq R'$ for any $k_{co}' \geq k_{co}$ and $k_{ru}' \geq k_{ru}$.
  \end{enumerate}
\end{proposition}

\begin{proof}
  1.~Let $\gsdefault(\assignment, \gmem, r) = (\gmemprime,R)$ and $r \in R$, thus by definition it holds that $r$ is not inactive and of interest w.r.t.~$\gmemprime = \assignmentp$. For the accumulator variant it holds that $\gmemdprime = \gmem \cup \assignmentp$ and $R' = \{ r \in \ground(r) \mid r \text{ is not inactive and}$ $\text{of interest w.r.t.~} \gmemdprime\}$. Since $\gmemprime \subseteq \gmemdprime$ and the assignment corresponding to a grounder memory is an assignment $\assignment$ such that $\assignmentp = \gmemprime$ and $\assignmentm = \emptyset$, it holds that $r$ is of interest w.r.t.~$\gmemdprime$, i.e., $r \in R'$.
  2.~and 3.~are analogous.
\end{proof}

\paragraph{Soundness and Completeness.} 
We show in the following that all grounding strategies are sound and complete, i.e., in a lazy-grounding ASP solver one may freely exchange one grounding strategy for another.

\begin{proposition} Given a lazy-grounding ASP solver $S$ which is sound and complete for the default grounding strategy $\gsdefault$, then $S$ is sound and complete for the $k$-unassigned grounding strategies $\gskunassigned_{k_{co},k_{ru}}$, and their respective accumulator variants $\gsdefaultaccu$ and $\gskunassignedaccu_{k_{co},k_{ru}}$.
\end{proposition}

\begin{proof}
  Soundness immediately follows from the respective definition, because every ground rule returned by any of the above grounding strategies is a ground rule of the original program. Formally, let $(\gmemprime,R)$ be the returned pair of any of these strategies then for all $r \in R$ holds that $r \in \ground(r)$ and thus $r \in \ground(P)$ where $P$ is the input program.

  Completeness: $S$ is complete for $\gsdefault$, intuitively if a ground rule $r$ fires under some assignment $\assignment$ then $r$ is of interest w.r.t.~$\assignment$ and hence returned by $\gsdefault$.
  Observe that a rule that is inactive can never be applicable in any answer set, hence the additional requirement to only consider rules that are not inactive has no effect on completeness.
  Completeness for all other grounding strategies then follows from \cref{prop:gskunassigned_subsumes_greaterk,prop:gskunassigned_subsumes_gsdefault,prop:gsaccu_subsumes_gs}, showing that every other grounding strategy produces at least the same ground rules as $\gsdefault$.
\end{proof}

  The lazy-grounding strategies $\gsdefault$,$\gskunassigned_{k_{co},k_{ru}}$, $\gsdefaultaccu$, and $\gskunassignedaccu_{k_{co},k_{ru}}$ are sound and complete for \alphaslv, since \alphaslv is sound and complete for $\gsdefault$ (cf.~\cite{alpha_technical}).

\paragraph{The Effect of Domain Predicates.}
It is well-known for practical ASP solving that the choice of encoding employed for a task can have a major influence on solving performance, even though the semantics is still declarative.
Such an effect can also be observed in conjunction with grounding strategies based on $k$-unassignedness.
Assume that $\mi{dom}$ is a domain predicate in the sense of \cite{DBLP:conf/lpnmr/Syrjanen01}, i.e. a predicate defining the domain over which $p$ and $q$ are defined, and consider the constraint $c$ as follows: $\leftarrow p(X), q(Y).$
If $p(1) \in \assignmentp$ and $q(t) \notin \assignmentp$ holds for all terms $t$ then $c$ is not $1$-unassigned, because $Y$ is not yet bound and thus $c$ is not all-variable-assigning.
Extending $c$ with domain predicates to obtain $c'$ gives $\leftarrow \mi{dom}(X), \mi{dom}(Y), p(X), q(Y).$ Assuming that $\mi{dom}(1)$ holds together with $p(1) \in \assignmentp$ and $q(1) \notin \assignmentp$ then yields the ground rule $\leftarrow \mi{dom}(1), \mi{dom}(1), p(1), q(1)$ which is $1$-unassigned w.r.t.~$\assignmentp$.
In such a case, the $1$-unassigned lazy-grounding strategy yields no ground instances for $c$ but some for $c'$.
Hence an earlier grounding of constraints (and rules) can be initiated by adding (superfluous) domain predicates.

Adding domain predicates allows finding a solution with fewer backtracks, because the additional ground constraints support early propagation and inform the search heuristics better.
This is not a guaranteed improvement, however, since more ground constraints also need more space.
A grounder can add domain predicates automatically
or use the heads of previously grounded rules to generate bindings
even if those heads are not true yet.
But this is future work.

\section{Experimental Results}
\label{sec:experiments}

To evaluate the novel grounding strategies an experimental study was carried out using two benchmark problems, Graph Colouring and House Reconfiguration.

\paragraph{Experimental Setup.}
Experiments were run on a cluster of machines each with two
Intel\textsuperscript{\textregistered} Xeon\textsuperscript{\textregistered} CPU E5-2650 v4 @ 2.20GHz with 12 cores each, 252 GB of memory, and Ubuntu 16.04.1 LTS Linux.
Benchmarks were scheduled with the ABC Benchmarking System\footnote{\url{https://github.com/credl/abcbenchmarking}} \cite{DBLP:conf/aiia/Redl16} together with HTCondor\textsuperscript{\texttrademark}.\footnote{\url{http://research.cs.wisc.edu/htcondor}}
Time and memory consumption were measured by \slv{pyrunlim},\footnote{\url{https://alviano.com/software/pyrunlim/}}
which was also used to limit time consumption to 15 minutes per instance and swapping to 0.

\paragraph{Encodings and Instances.}
The encoding for Graph Colouring was taken from the Fourth Answer Set Programming Competition \cite{Alviano.2013},
with a choice rule replacing the original disjunctive rule without altering semantics of the problem.
The encoding for the House Reconfiguration Problem was taken from \cite{Ryabokon.2015}, but changed to a decision problem, since optimization is not yet supported by \alphaslv.

Problem instances from the ASP Competitions \cite{Alviano.2013,Calimeri.2016} decidedly were not used, because these are hand picked to exercise search techniques of ground-and-solve systems, some of which are not (yet) available in lazy-grounding ASP solving, like restarts and equivalence preprocessing.\footnote{Graph Colouring benchmark instances, for example, are prohibitive even for \clasp\ with those techniques disabled by \texttt{--sat-prepro=no --eq=0 -r no -d no}.}

For Graph Colouring, Erdős–Rényi graphs \cite{Erdos.1959} were generated\footnote{Using the Python function \texttt{networkx.generators.random\_graphs.gnm\_random\_graph}.}.
Let $(V,E,C)$ denote a class of Graph Colouring instances, where $V$ denotes the number of vertices, $E$ the number of edges, and $C$ the number of colours.
For each configuration in the set $\{ (V,E,C) \mid V \in \{ 10, 20, \dots, 190, 200, 250, \dots, 450, 500 \},$ $\frac{E}{V} \in \{ 4, 8, 16 \}, C \in \{ 3, 5 \} \} \setminus (\{ (V,E,5) \mid V \geq 180, \frac{E}{V} = 16 \} \cup \{ (V,E,5) \mid V \geq 100, \frac{E}{V} = 8 \})$,
11 graphs were generated.
This makes 1430 instances in total.
The values for $E$ and $C$ were chosen to obtain a diverse set of instances based on values used for the ASP competitions.

Instances for the House Reconfiguration Problem were also generated randomly.
For each number of things $T \in \{ 5, 10, \dots, 40, 45 \}$, 11 instances were generated.
This makes 99 instances in total.
For each instance, the number of persons $P$ was drawn from a uniform distribution $\mathcal{U}\{2, \lfloor \frac{T}{2} \rfloor + 1\}$ and the owner of every thing was drawn from $\mathcal{U}\{1, P\}$.
Every thing had a $50\%$ chance to be in a cabinet, which was then drawn from a uniform distribution.
A random subset of given things was considered as long things, the cardinality $T_{long}$ of which was drawn from a normal distribution and then fit into the available range $\{0, \dots, T\}$ by computing $T_{long} = \mi{min}(T,|n|)$ after drawing $n$ from $\mathcal{N}(0,T^2)$.

\paragraph{Results and Discussion.}
Using the method described above, we obtained data on \alphaslv's resource consumption for processing the benchmarks.
For each instance, \alphaslv\ was instructed to find 10 answer sets.
To reduce the numbers of data points in the scatter plots in this section, we do not show every problem instance, but the median performance data for each size and class of problem instances.\footnote{Computing the median of an odd number of performance data allows to obtain a measure of central tendency that is unaffected by timeouts.}
For example, all 11 Graph Colouring instances of each class $(V,E,C)$ are condensed into one data point for each pair of grounding strategies compared by the plot.

\begin{figure}[t]
\centering
\begin{subfigure}{.48\textwidth}
  \centering
  \includegraphics[width=\linewidth]{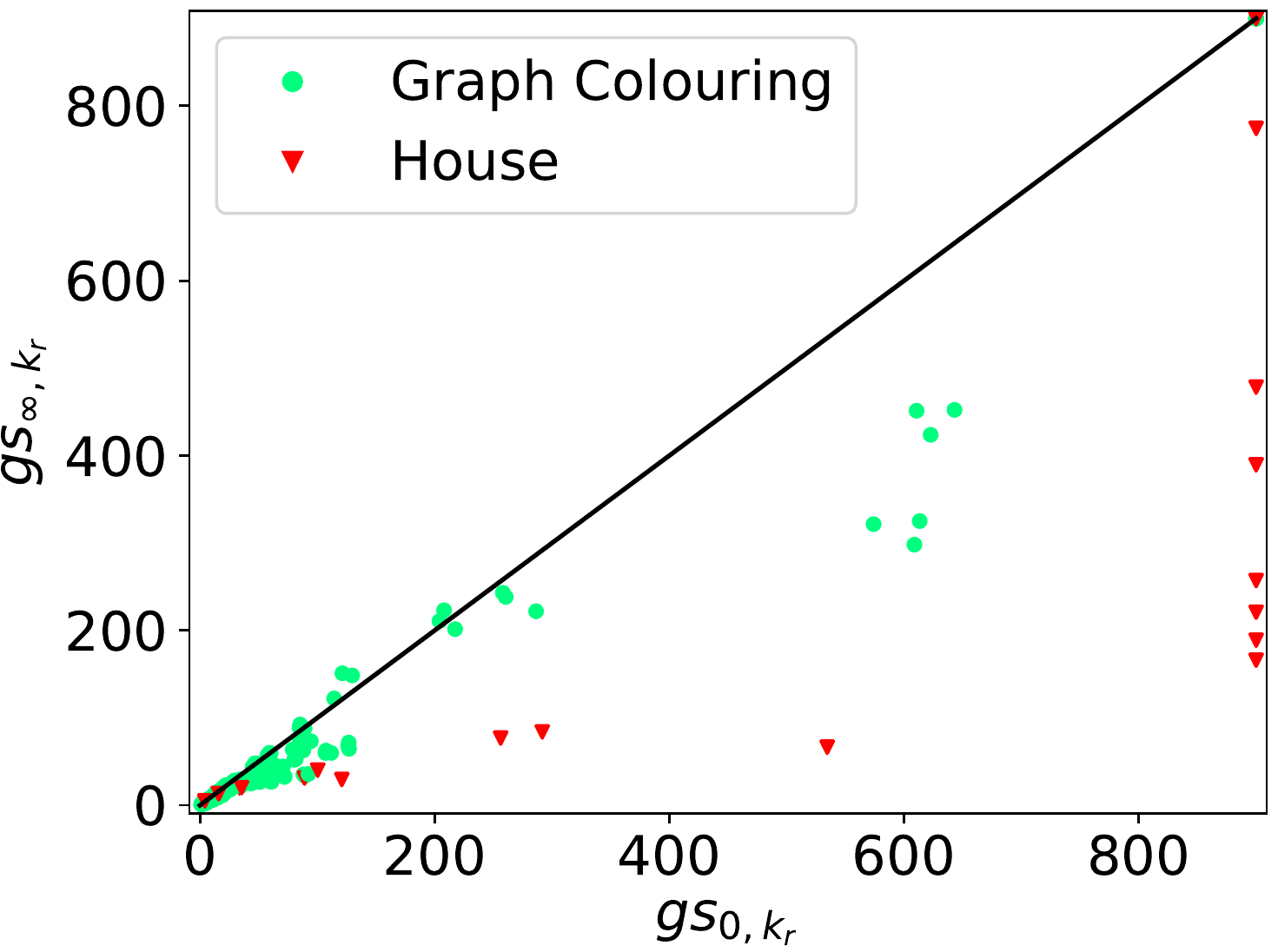}
  \caption{Time consumption (s)}
  \label{fig:scatter_constraints_l_vs_s_all_problems_time}
\end{subfigure}%
\hspace*{\fill}
\begin{subfigure}{.48\textwidth}
  \centering
  \includegraphics[width=\linewidth]{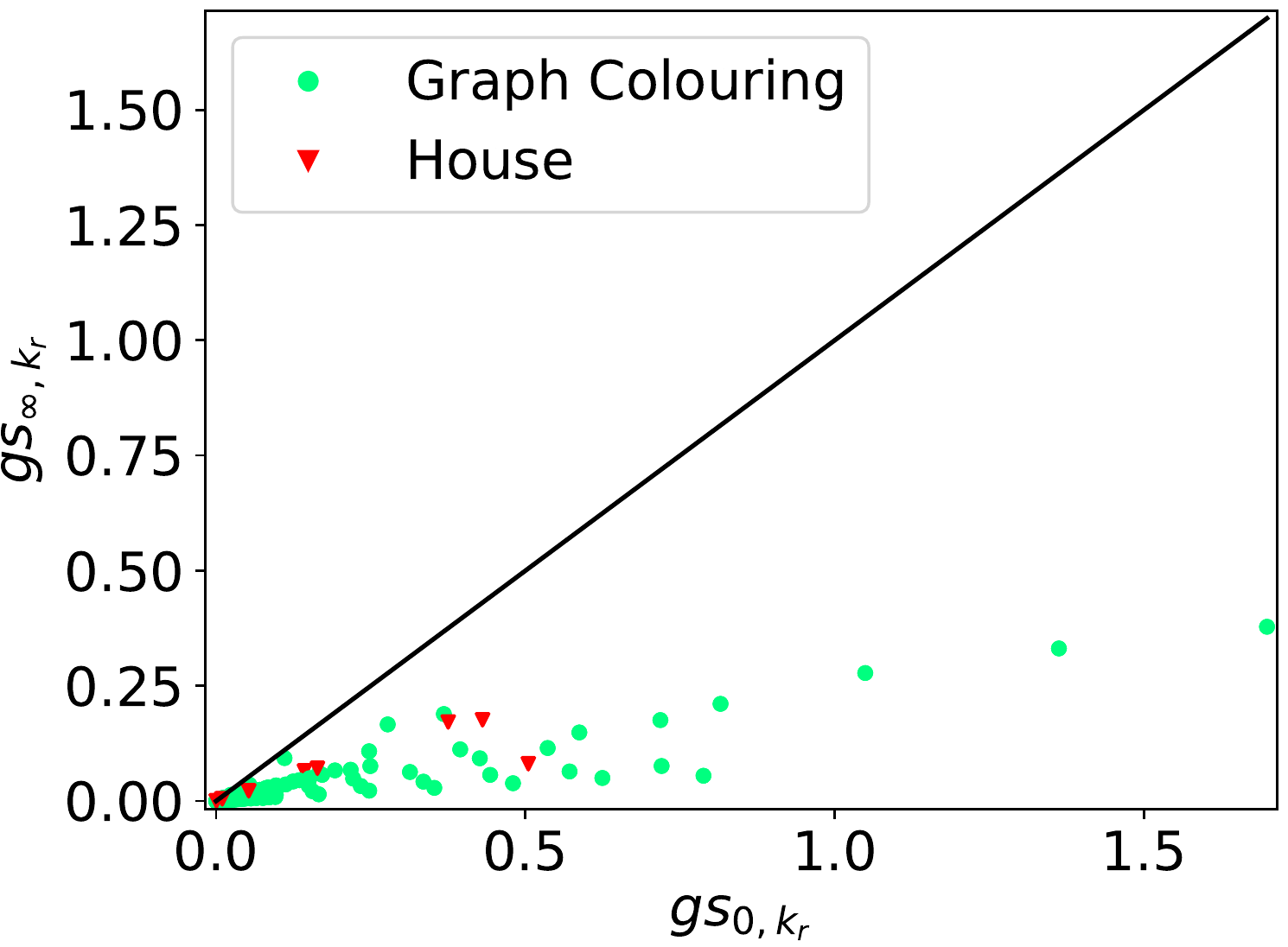}
  \caption{Guesses times $10^6$}
  \label{fig:scatter_constraints_l_vs_s_all_problems_g}
\end{subfigure}
\caption{Time and guesses comparing $\gskunassigned_{0,k_{ru}}$ to $\gskunassigned_{\infty,k_{ru}}$ for $k_{ru} \in \{ 0, 1, \infty \}$}
\label{fig:scatter_constraints_l_vs_s_all_problems_time_g}
\end{figure}

\cref{fig:scatter_constraints_l_vs_s_all_problems_time_g} shows the resource usage needed to find the first 10 answer sets of each benchmark instance, comparing strict ($k_{co} = 0$) to permissive ($k_{co} = \infty$) grounding of constraints.
Each data point in the scatter plots corresponds to one class of problem instances of the same size solved by two different grounder configurations for $k_{co} \in \{ 0, \infty \}$ and varying $k_{ru} \in \{ 0, 1, \infty \}$.
The location of each data point on the $x$ axis corresponds to resource usage with $k_{co} = 0$, its $y$ location to resource usage with $k_{co} = \infty$.
Hence, a data point on the diagonal corresponds to a problem instance where strict ($k_{co} = 0$) and permissive ($k_{co} = \infty$) perform equally well.
Data points that are located below the diagonal indicate that an instance could be solved faster when using $k_{co} = \infty$, while those above the diagonal represent an instance that could be solved faster when using $k_{co} = 0$.
Instances that exceeded the given time-out of 900 seconds line up at the end of each axis.
Time usage is shown in \cref{fig:scatter_constraints_l_vs_s_all_problems_time}, number of guesses in \cref{fig:scatter_constraints_l_vs_s_all_problems_g}.\footnote{Numbers of guesses are only shown for instances that could be solved within the given time limit.}

Since most data points in \cref{fig:scatter_constraints_l_vs_s_all_problems_time_g} are located below the diagonal, it is evident that permissive grounding of constraints led to faster solving in most cases of Graph Colouring and all cases of House Configuration.
Comparing \cref{fig:scatter_constraints_l_vs_s_all_problems_time,fig:scatter_constraints_l_vs_s_all_problems_g} shows an even greater advantage of permissive grounding
when the number of guesses is considered.

When strict and permissive settings for $k_{ru}$ instead of $k_{co}$ are compared, no clear conclusion can be drawn which value yields the best performance.
Due to space constraints, the corresponding plots are not shown.

\begin{figure}[t]
\centering
\begin{subfigure}{.48\textwidth}
  \centering
  \includegraphics[width=\linewidth]{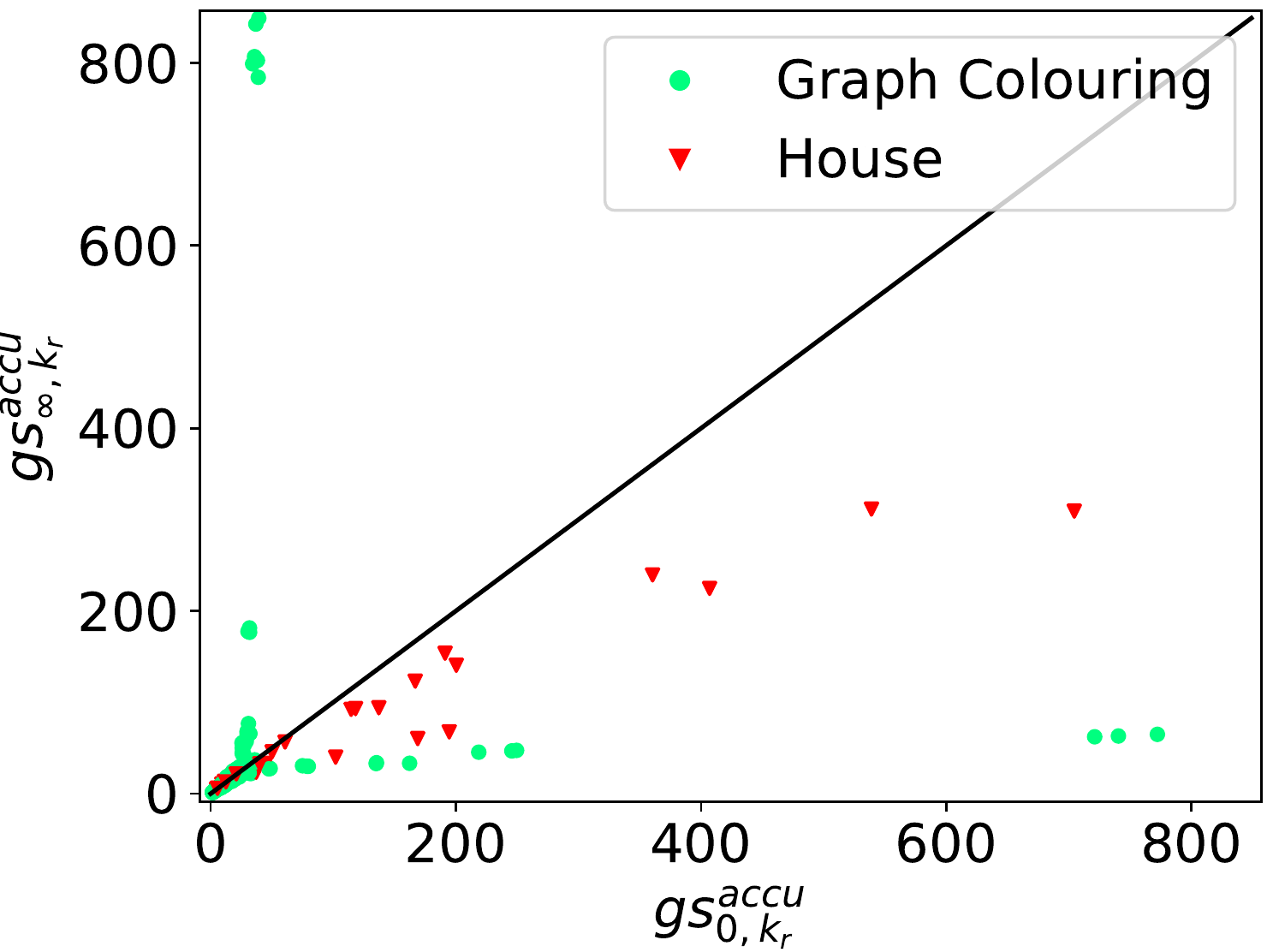}
  \caption{Time consumption (s)}
  \label{fig:scatter_constraints_l_vs_s_all_problems_time_accum}
\end{subfigure}%
\hspace*{\fill}
\begin{subfigure}{.48\textwidth}
  \centering
  \includegraphics[width=\linewidth]{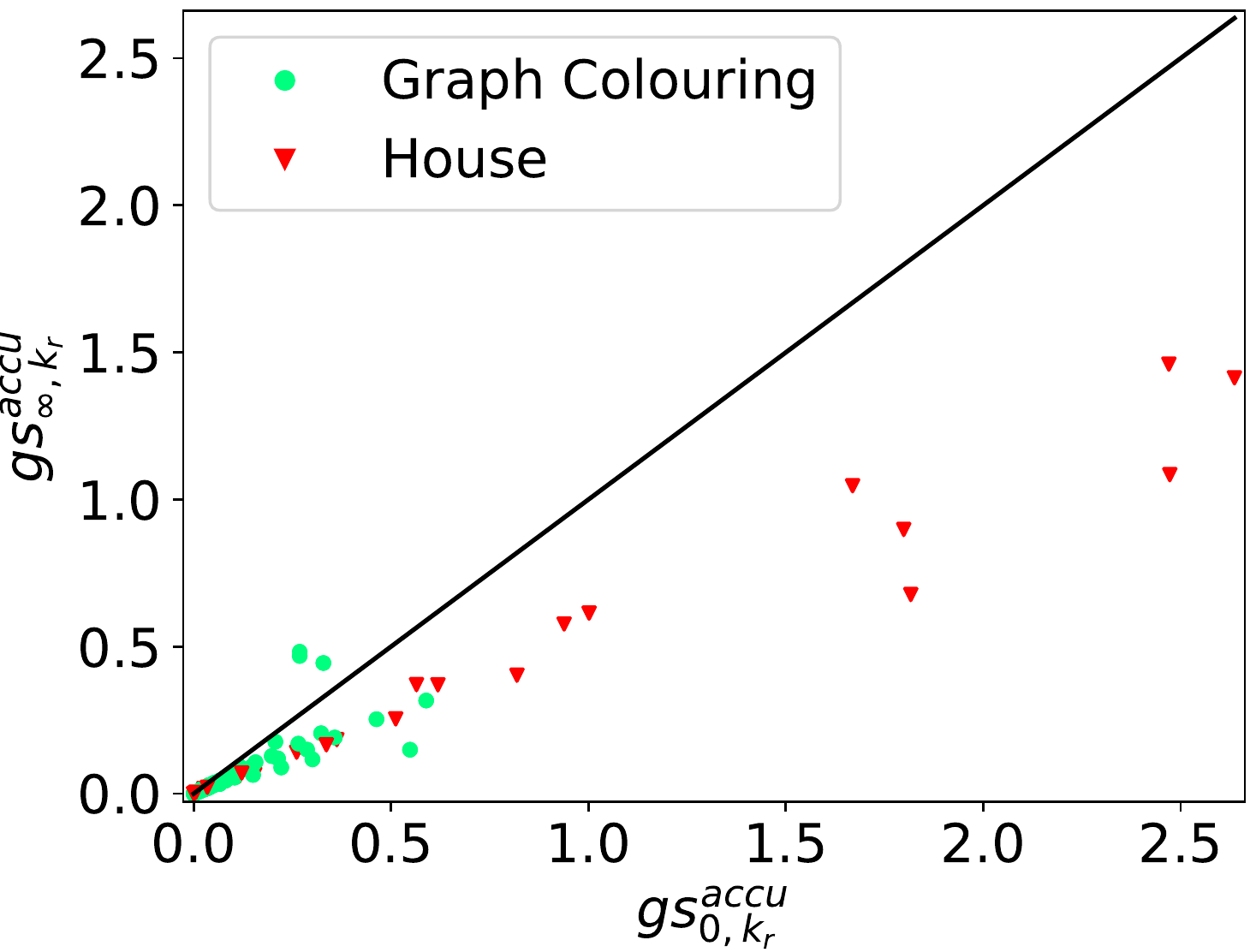}
  \caption{Guesses times $10^6$}
  \label{fig:scatter_constraints_l_vs_s_all_problems_g_accum}
\end{subfigure}
\caption{Time and guesses comparing $\gskunassignedaccu_{0,k_{ru}}$ to $\gskunassignedaccu_{\infty,k_{ru}}$ for $k_{ru} \in \{ 0, 1, \infty \}$}
\label{fig:scatter_constraints_l_vs_s_all_problems_time_g_accum}
\end{figure}

\cref{fig:scatter_constraints_l_vs_s_all_problems_time_g_accum} shows the same instances for the accumulator variants of the same grounding strategies.
Again, the general pattern indicates that permissive grounding of constraints ($k_{co}= \infty$) improves performance.
Comparing those plots to \cref{fig:scatter_constraints_l_vs_s_all_problems_time_g} shows a noticeable change of the performance in Graph Colouring instances.
Most of the data points gather along the diagonal near the origin and a small (but more visible) number
of outliers is distributed near both axes, which means that some Graph Colouring instances were hard to solve for $k_{co} = 0$ and some were hard to solve for $k_{co} = \infty$.
Deeper analysis of the solver revealed that in these cases the branching heuristic completely leads the search astray, resulting in more guesses to solve the problem and to invest more time in propagation.

\begin{figure}[t]
\centering
\begin{subfigure}{.48\textwidth}
  \centering
  \includegraphics[width=\linewidth]{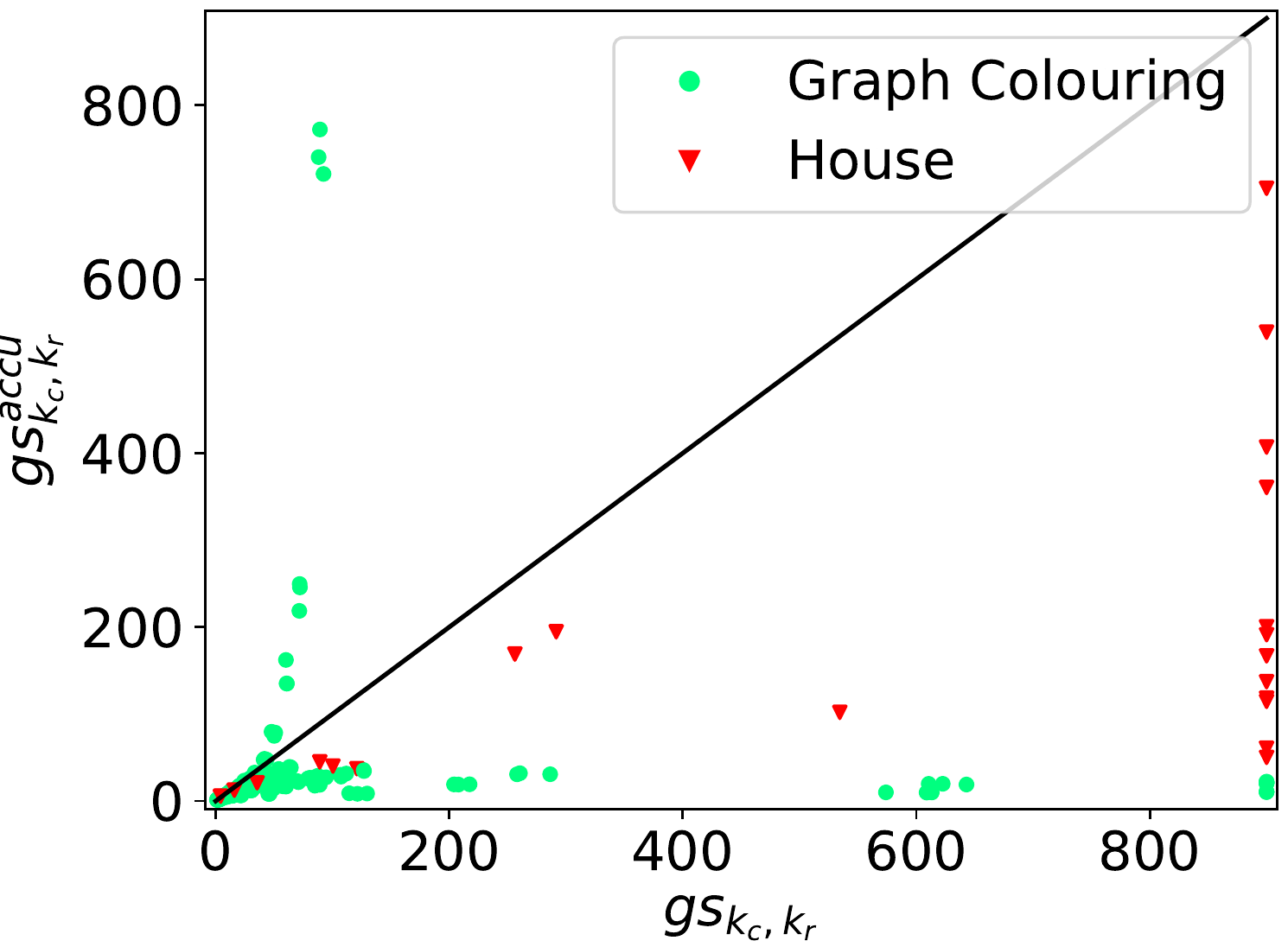}
  \caption{Time consumption (s)}
  \label{fig:scatter_accum_vs_default_all_problems_time}
\end{subfigure}%
\hspace*{\fill}
\begin{subfigure}{.48\textwidth}
  \centering
  \includegraphics[width=\linewidth]{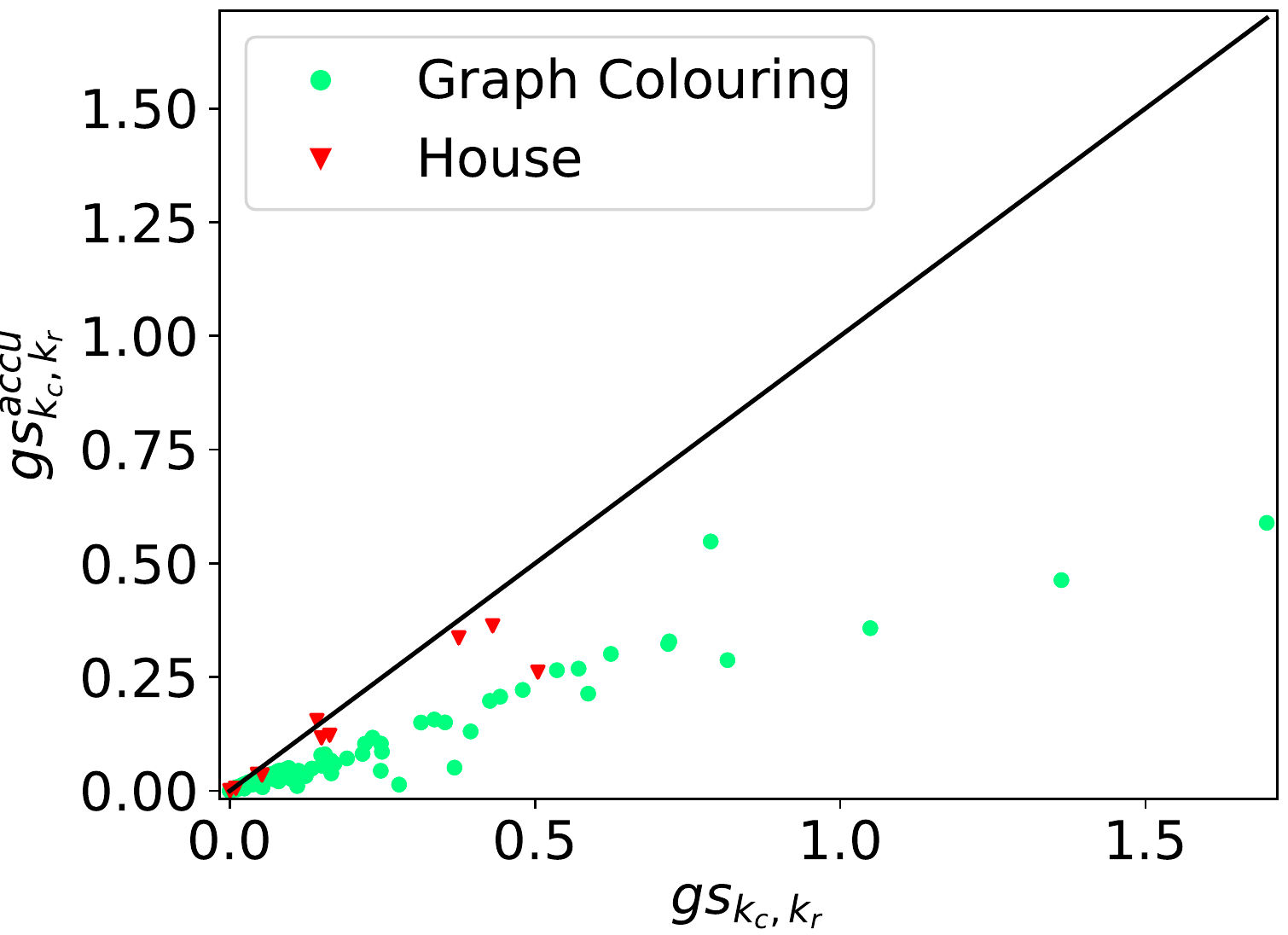}
  \caption{Guesses times $10^6$}
  \label{fig:scatter_accum_vs_default_all_problems_g}
\end{subfigure}
\caption{Time and guesses comparing $\gskunassigned_{k_{co},k_{ru}}$ to $\gskunassignedaccu_{k_{co},k_{ru}}$ for $k_{co}, k_{ru} \in \{ 0, 1, \infty \}$}
\label{fig:scatter_accum_vs_default_all_problems_g_time}
\end{figure}

In \cref{fig:scatter_accum_vs_default_all_problems_g_time}, we compare results for accumulator grounding strategies to their variants without accumulator.
We observe a similar pattern as in \cref{fig:scatter_constraints_l_vs_s_all_problems_time_g_accum}:
while House is clearly able to benefit
from the accumulator, effects are mixed for Graph Colouring. Visually, outliers dominate the plot but most data points are near the origin and below the diagonal.

\begin{figure}[t]
\centering
\begin{subfigure}{.48\textwidth}
  \centering
  \includegraphics[width=\linewidth]{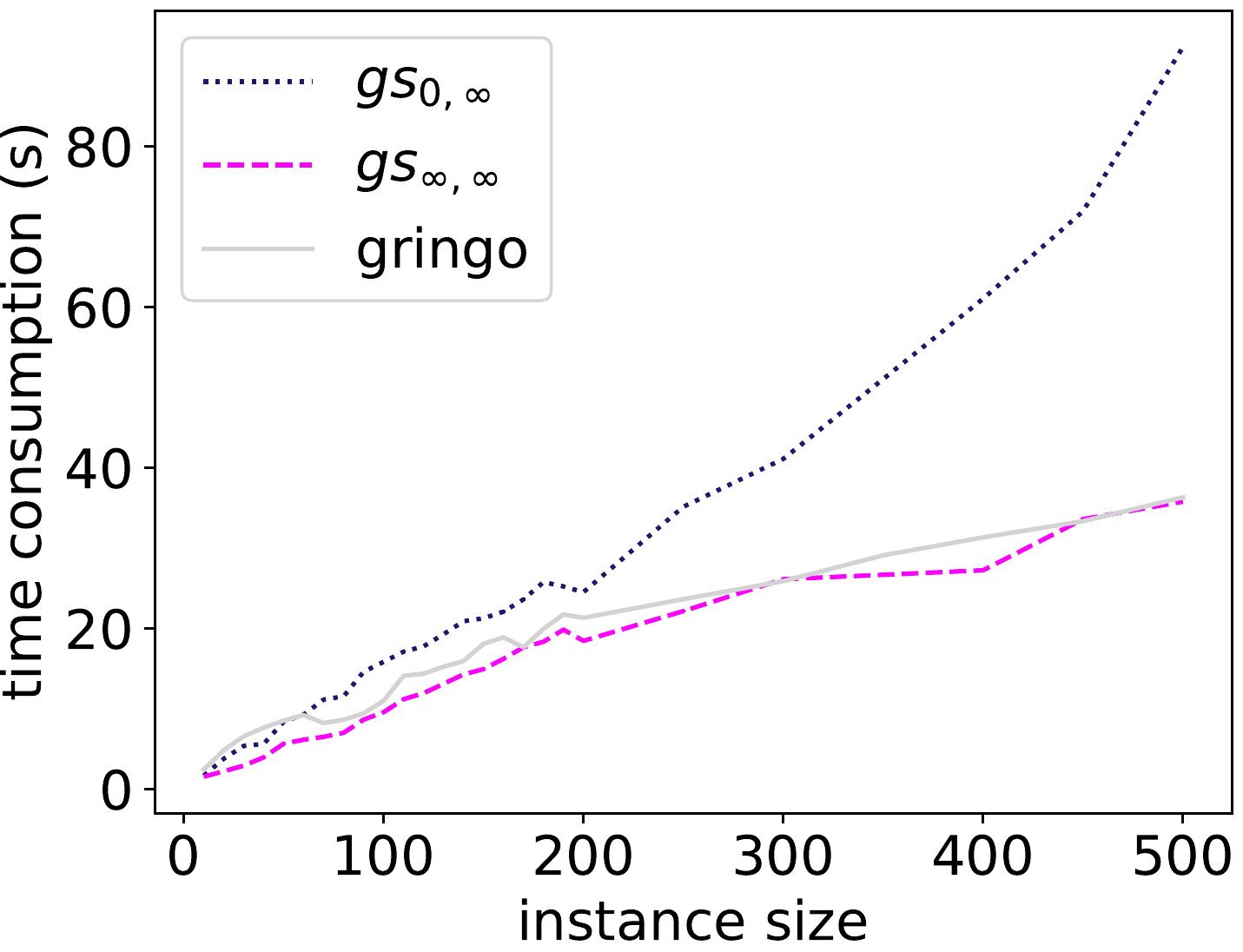}
  \caption{$\frac{E}{V} \approx 16$, $C = 3$}
  \label{fig:line_grounderconfigs_GraphColouring_ev16_c3_time}
\end{subfigure}%
\hspace*{\fill}
\begin{subfigure}{.48\textwidth}
  \centering
  \includegraphics[width=\linewidth]{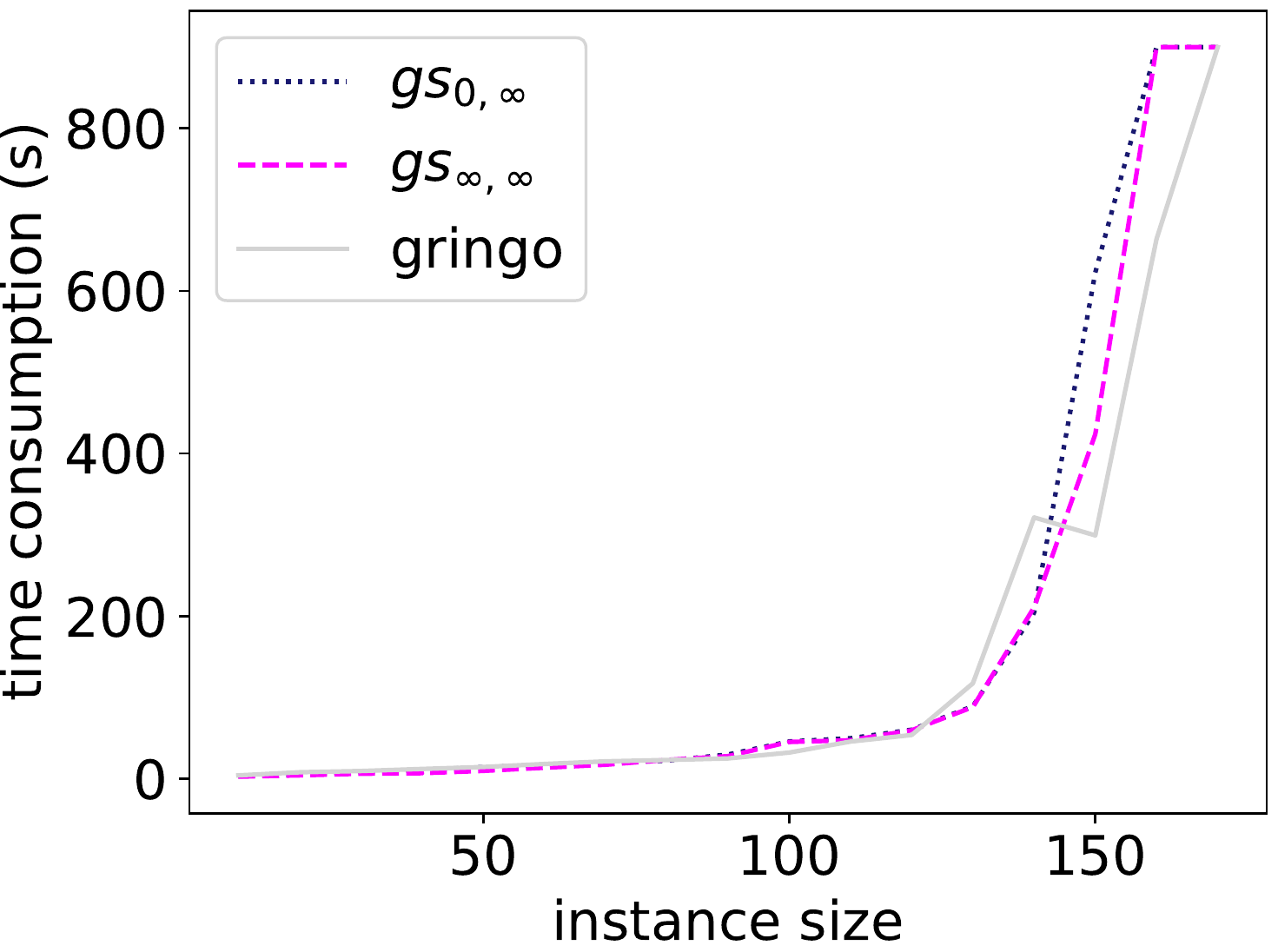}
  \caption{$\frac{E}{V} \approx 16$, $C = 5$}
  \label{fig:line_grounderconfigs_GraphColouring_ev16_c5_time}
\end{subfigure}
\caption{Time consumption for two classes of Graph Colouring instances (without accumulator)}
\label{fig:line_grounderconfigs_GraphColouring_ev16_c3_c5_time}
\end{figure}

\Cref{fig:line_grounderconfigs_GraphColouring_ev16_c3_c5_time} offers a different perspective on time consumption data for two classes of Graph Colouring instances.
For each instance size (number of nodes), the median time consumption of selected grounding strategies on all 11 instances is plotted.
From the range of lazy grounding strategies, two representatives are shown, comparing strict to permissive grounding of constraints.\footnote{In Graph Colouring, changing the value of $k_{ru}$ does not show an effect on the number of guesses needed to find an answer set.
  This is likely due to the encoding containing only one non-constraint rule whose body is not fully determined by facts.}
These are contrasted with \alphaslv's performance when working on a fully ground input produced by \slv{gringo} version 5.2.2.
It appears that instances with $\frac{E}{V} \approx 16$ and three colours were able to benefit greatly from permissive grounding of constraints, while performance was rather unaffected by change of grounding strategies when five colours were used instead. For both classes we observe that having the full upfront grounding provides good performance compared to lazy grounding, which is in line with heuristics being fully informed.
The $k$-unassigned grounding strategy with permissive grounding of constraints, however, performs similarly well.
We observed that permissive grounding of constraints also reduces memory usage.
We assume this to be caused by the current lack of learned nogood forgetting, i.e.~the longer \alphaslv runs the more learned nogoods are kept in memory.
Due to space constraints no data on memory consumption is shown here.

Overall, we observe that lazy grounding enables a whole new range of lazy-grounding strategies that face other challenges than previous approaches at grounding.
Most importantly, in lazy grounding
rules and constraints grounded earlier than necessary have a great effect on solving performance, because they inform the heuristics about the search space.
While we cannot give a definite answer on which grounding strategy is the best, we uncovered a whole new field of possible strategies and identified some that improve efficiency significantly.

\section{Conclusions and Future Work}
\label{sec:conclusion}

In this work we introduced a field of novel grounding strategies for lazy-grounding ASP evaluation.
Grounding lazily as little as possible adversely affects heuristics and search performance, because of the limited view of the search space.
Our investigation aimed at new ways to offset this restriction while keeping the benefits of lazy grounding to avoid the grounding bottleneck.
The main contribution of this paper is the introduction and formal characterization of various classes of grounding strategies (\enquote{degrees of laziness}), like $k$-unassigned grounding strategies and accumulator-based ones, which allow compromises between lazily grounding as little as possible and the traditional grounding upfront.
Experimental results show a clear improvement over existing lazy-grounding strategies and that permissive grounding of constraints usually improves solving performance, while the performance improvements from other grounding strategies depend on the problem to be solved.
Permissive lazy grounding of constraints could become the new default for \alphaslv\ and may be applied in other lazy-grounding solvers.

Our work considers grounding from a very different (lazy) perspective than previous works on (upfront) grounding.
As such, it cannot provide the conclusion but rather the beginning of a larger investigation on the effects of lazy-grounding strategies on solving performance.
Future work may explore syntactic features of answer-set programs to automatically select an efficient grounding strategy and investigate connecting lazy grounding more closely with search heuristics.

\subsection*{Acknowledgements}
This work has been conducted in the scope of the research project \textit{DynaCon (FFG-PNr.: 861263)}, which is funded by the Austrian Federal Ministry of Transport, Innovation and Technology (BMVIT) under the program ``ICT of the Future" between 2017 and 2020 (see \url{https://iktderzukunft.at/en/} for more information).

%
%
%
\bibliography{bibliography}

\begin{thebibliography}{10}
\providecommand{\url}[1]{\texttt{#1}}
\providecommand{\urlprefix}{URL }
\providecommand{\doi}[1]{https://doi.org/#1}

\bibitem{Alviano.2013}
Alviano, M., Calimeri, F., Charwat, G., Dao{-}Tran, M., Dodaro, C., Ianni, G.,
  Krennwallner, T., Kronegger, M., Oetsch, J., Pfandler, A., P{\"{u}}hrer, J.,
  Redl, C., Ricca, F., Schneider, P., Schwengerer, M., Spendier, L.K., Wallner,
  J.P., Xiao, G.: The fourth answer set programming competition: Preliminary
  report. In: {LPNMR}. LNCS, vol.~8148, pp. 42--53. Springer (2013)

\bibitem{aspbook-baral}
Baral, C.: {Knowledge Representation, Reasoning and Declarative Problem
  Solving}. Cambridge University Press (2003)

\bibitem{DBLP:journals/ia/CalimeriFPZ17}
Calimeri, F., Fusc{\`{a}}, D., Perri, S., Zangari, J.: {I-DLV:} the new
  intelligent grounder of {DLV}. Intelligenza Artificiale  \textbf{11}(1),
  5--20 (2017)

\bibitem{Calimeri.2016}
Calimeri, F., Gebser, M., Maratea, M., Ricca, F.: Design and results of the
  fifth answer set programming competition. Artif. Intell.  \textbf{231},
  151--181 (2016)

\bibitem{omiga_system}
Dao{-}Tran, M., Eiter, T., Fink, M., Weidinger, G., Weinzierl, A.: Omiga : An
  open minded grounding on-the-fly answer set solver. In: {JELIA}. LNCS,
  vol.~7519, pp. 480--483. Springer (2012)

\bibitem{Eiter.2008}
Eiter, T., Faber, W., Fink, M., Woltran, S.: Complexity results for answer set
  programming with bounded predicate arities and implications. Ann. Math.
  Artif. Intell.  \textbf{51}(2-4),  123--165 (2007)

\bibitem{Erdos.1959}
Erdős, P., Rényi, A.: {On Random Graphs. I}. {Publicationes Mathematicae}
  \textbf{6},  290--297 (1959),
  \url{https://users.renyi.hu/~p_erdos/1959-11.pdf}

\bibitem{DBLP:conf/lpnmr/GebserKK0S15}
Gebser, M., Kaminski, R., Kaufmann, B., Romero, J., Schaub, T.: Progress in
  clasp series 3. In: {LPNMR}. LNCS, vol.~9345, pp. 368--383. Springer (2015)

\bibitem{Gebser.2012}
Gebser, M., Kaminski, R., Kaufmann, B., Schaub, T.: {Answer Set Solving in
  Practice}. {Morgan and Claypool Publishers} (2012)

\bibitem{DBLP:conf/lpnmr/GebserKKS11}
Gebser, M., Kaminski, R., K{\"{o}}nig, A., Schaub, T.: Advances in
  \emph{gringo} series 3. In: {LPNMR}. LNCS, vol.~6645, pp. 345--351. Springer
  (2011)

\bibitem{clasp_journal}
Gebser, M., Kaufmann, B., Schaub, T.: Conflict-driven answer set solving: From
  theory to practice. Artif. Intell.  \textbf{187},  52--89 (2012)

\bibitem{aspbook-gelfond}
Gelfond, M., Kahl, Y.: Knowledge Representation, Reasoning, and the Design of
  Intelligent Agents: The Answer-Set Programming Approach. Cambridge University
  Press, New York, NY, USA (2014)

\bibitem{Gelfond.1988}
Gelfond, M., Lifschitz, V.: The stable model semantics for logic programming.
  In: {ICLP/SLP}. pp. 1070--1080. {MIT} Press (1988)

\bibitem{Lefevre.2017}
Lef{\`{e}}vre, C., B{\'{e}}atrix, C., St{\'{e}}phan, I., Garcia, L.: Asperix, a
  first-order forward chaining approach for answer set computing. {TPLP}
  \textbf{17}(3),  266--310 (2017)

\bibitem{Leone.2006}
Leone, N., Pfeifer, G., Faber, W., Eiter, T., Gottlob, G., Perri, S.,
  Scarcello, F.: The {DLV} system for knowledge representation and reasoning.
  {ACM} Trans. Comput. Log.  \textbf{7}(3),  499--562 (2006)

\bibitem{Leutgeb.2017}
Leutgeb, L., Weinzierl, A.: Techniques for efficient lazy-grounding {ASP}
  solving. In: {DECLARE}. LNCS, vol. 10997, pp. 132--148. Springer (2017)

\bibitem{Moskewicz.2001}
Moskewicz, M.W., Madigan, C.F., Zhao, Y., Zhang, L., Malik, S.: Chaff:
  Engineering an efficient {SAT} solver. In: {DAC}. pp. 530--535. {ACM} (2001)

\bibitem{gasp}
Pal{\`{u}}, A.D., Dovier, A., Pontelli, E., Rossi, G.: {GASP:} answer set
  programming with lazy grounding. Fundam. Inform.  \textbf{96}(3),  297--322
  (2009)

\bibitem{Pretolani.1993}
Pretolani, D.: Efficiency, and stability of hypergraph {SAT} algorithms. In:
  Cliques, Coloring, and Satisfiability. vol.~26, pp. 479--498. {DIMACS/AMS}
  (1993)

\bibitem{DBLP:conf/aiia/Redl16}
Redl, C.: Automated benchmarking of {KR}-systems. In: RCRA@AI*IA. {CEUR}
  Workshop Proceedings, vol.~1745, pp. 45--56. CEUR-WS.org (2016)

\bibitem{Ryabokon.2015}
Ryabokon, A.: {Knowledge-Based (Re)Configuration of Complex Products and
  Services}. {Dissertation}, {Alpen-Adria-Universität Klagenfurt}, Klagenfurt
  (2015)

\bibitem{DBLP:conf/lpnmr/Syrjanen01}
Syrj{\"{a}}nen, T.: Omega-restricted logic programs. In: {LPNMR}. Lecture Notes
  in Computer Science, vol.~2173, pp. 267--279. Springer (2001)

\bibitem{alpha_technical}
Weinzierl, A.: Blending lazy-grounding and {CDNL} search for answer-set
  solving. In: {LPNMR}. LNCS, vol. 10377, pp. 191--204. Springer (2017)

\end{thebibliography}
\bibliographystyle{splncs04}
\end{document}